\documentclass[pdflatex,sn-mathphys-num]{sn-jnl}

\usepackage{graphicx}
\usepackage{multirow}
\usepackage{amsmath,amssymb,amsfonts}
\usepackage{amsthm}
\usepackage{mathrsfs}
\usepackage[title]{appendix}
\usepackage{xcolor}
\usepackage{textcomp}
\usepackage{manyfoot}
\usepackage{booktabs}
\usepackage{algorithm}
\usepackage{algorithmicx}
\usepackage{algpseudocode}
\usepackage{listings}

\theoremstyle{thmstyleone}
\newtheorem{theorem}{Theorem}
\newtheorem{proposition}{Proposition}
\theoremstyle{thmstyletwo}
\newtheorem{example}{Example}
\newtheorem{remark}{Remark}

\theoremstyle{thmstylethree}

\begin{document}

\title[Identifiability of $k$-DPPs]{%
From DPPs to $k$-DPPs: identifiability analysis\\ via spectral decomposition}

\author{\fnm{Hideitsu} \sur{Hino}}\email{hino@ism.ac.jp}
\author{\fnm{Keisuke} \sur{Yano}}\email{yano@ism.ac.jp}

\affil{%
  \orgname{The Institute of Statistical Mathematics},
  \orgaddress{%
    \street{10-3 Midori cho},
    \city{Tachikawa City},
    \postcode{190-8562},
    \state{Tokyo},
    \country{Japan}}}

\abstract{%
We study the geometry of determinantal point processes (DPPs) through
the spectral decomposition $L=U\Lambda U^{\top}$.  The spectrum
$\Lambda$ governs the cardinality distribution via elementary symmetric
polynomials, while the eigenspace orientation $U$ governs the
conditional law within each fixed-cardinality stratum.  Conditioning
on cardinality $k$ yields the $k$-DPP, for which the identifiability
structure changes fundamentally: the spectral parameter becomes
identifiable only up to a common scale, and the eigenspace rotation parameter is
identifiable only through squared minors of the eigenvector matrix. 
We characterize the identifiability gap precisely, via three explicit
invariances (scale, sign similarity, and eigenspace rotation) and a
dimension-counting theorem showing the existence of additional continuous non-identifiability whenever $\binom{N}{k}<N(N+1)/2$.  In contrast, for the full
DPP the non-identifiability comes only from the discrete sign similarity.}


\maketitle

% =====================================================================
\section{Introduction}
% =====================================================================

Determinantal point processes (DPPs~\cite{KuleszaandTasker_2012})
provide a probabilistic framework for modeling repulsive interactions
and diversity.  Despite their algorithmic tractability, the parameter
space of DPPs possesses a highly nontrivial geometric structure.
Unlike flat exponential families~\cite{Amari1982-ts}, admissible
kernels are constrained by positivity, spectral bounds, and nonlinear
relations among principal minors. 
Understanding how these constraints
are organized geometrically is essential for interpreting
parametrization, inference, and structural restrictions of the model.

From the perspective of algebraic geometry,
the number of critical points of the DPP log-likelihood and the
algebraic equations governing its parameters are
elucidated~\cite{Friedman2024-wb}.  The asymptotic properties of
maximum likelihood estimators via an likelihood-geometric approach
are analyzed in~\cite{Brunel2017-iq}. 
Information-geometric quantities related to DPP has been also analyzed in~\cite{Hino2024-xm}.
These existing geometric analyses
have primarily relied on diagonal scaling, which clarify how DPPs are embedded into  larger exponential families~\cite{Hino2024-xm} but leave cardinality and
correlation inherently coupled.

The spectral decomposition $L=U\Lambda U^{\top}$ suggests an
alternative coordinate system in which two statistically distinct
components appear: the eigenvalues $\Lambda$ and the eigenspace
orientation $U$.  A key question is how conditioning on cardinality
$|A|=k$, which passes from the full DPP to the $k$-DPP~\cite{Kulesza2011-qo}, changes the identifiability of each component.

The contributions of this letter are the following.
\begin{enumerate}
\item \emph{Parameterization using the spectral decomposition.}
We characterize $k$-DPPs using the parametrization from the spectral decomposition. When the eigenspace orientation is fixed to the identity matrix, it becomes an exponential family.
In this case, the degeneracy of the Fisher information matrix is analyzed. The Fisher information matrix degenerates exactly along the
common-shift direction $\mathbb{R}\mathbf{1}_{N}:=\{c\mathbf{1}_{N}\;:\;c\in\mathbb{R}\}$ with $\mathbf{1}_{N}$ the all-one vector of size $N$, and at symmetric
points reduces to the canonical metric on the additive quotient space
$\mathbb{R}^N/\mathbb{R}\mathbf{1}_{N}$.

\item \emph{Identifiability of $k$-DPPs.}
We characterize the identifiability gap between the full DPP and the
$k$-DPP.  Three explicit invariances are established (scale, sign
similarity, eigenspace rotation), and a dimension-counting theorem
shows the existence of additional continuous non-identifiability whenever
$\binom{N}{k}<N(N+1)/2$.
\end{enumerate}

% =====================================================================
\section{Spectral decomposition of DPPs}
% =====================================================================

Let $\mathcal{Y}:=\{1,\ldots,N\}$.
Let $\mathrm{Sym}^+(N,\mathbb{R})$ be the set of real-valued, symmetric, positive semidefinite $N\times N$ matrices.
A DPP on $\mathcal{Y}$ assigns to each subset
$A\subseteq \mathcal{Y}$ the probability
\[
P_L(A)=\frac{\det(L_A)}{\det(L+I_{N})},
\]
where $L\in\mathrm{Sym}^+(N,\mathbb{R})$ is the kernel matrix and $I_{N}$ is the $N\times N$ identity matrix.  DPPs
form a curved exponential family embedded in the space of log-linear
models~\cite{Hino2024-xm}.

The sample space $2^\mathcal{Y}$ is partitioned into cardinality strata
\[
2^\mathcal{Y}=\bigsqcup_{k=0}^{N}\mathcal{X}_k,
\qquad
\mathcal{X}_k:=\{A\subseteq \mathcal{Y}:|A|=k\},
\]
and the probability of $A$ factorizes as
$P_L(A)=P(|A|=k)\,P(A\mid A \in \mathcal{X}_{k})$,
where $P(A\mid A \in \mathcal{X}_{k})$ is the $k$-DPP \cite{Kulesza2011-qo}.

Let $L=U\Lambda U^{\top}$ with $\Lambda=\mathrm{diag}(\lambda_1,\ldots,\lambda_N)$ and $U\in O(N)$
be the spectral decomposition,
where $O(N)$ is the set of $N\times N$ orthogonal matrices.
The spectrum $\Lambda$ determines the
cardinality distribution~\cite{Kulesza2011-qo}:
\begin{equation}
P(|A|=k)
=\frac{e_k(\lambda_1,\ldots,\lambda_N)}{\prod_{i=1}^{N}(1+\lambda_i)},
\label{eq:Pk}
\end{equation}
where $e_k$ is the $k$-th elementary symmetric polynomial, $e_k(\lambda_1,\dots,\lambda_N)
:= \sum_{1\le i_1<\cdots<i_k\le N}
\lambda_{i_1}\cdots \lambda_{i_k}$. The
eigenspace orientation $U\in O(N)$, together with $\Lambda$, determines the within-stratum law:
\begin{equation}
P(A\mid A\in \mathcal{X}_{k})
\propto\det\bigl((U\Lambda U^{\top})_A\bigr).
\label{eq:kDPP}
\end{equation}
Thus the spectral decomposition separates cardinality control from
within-cardinality-stratum correlation structure.

In what follow, we write $P^k_L$ for the probability measure of the $k$-DPP with kernel $L$, i.e. $P^k_L(A) = \det(L_A)/Z_k(L)
$ for $A\in\mathcal{X}_{k}$, where $Z_k(L):=\sum_{A\in \mathcal{X}_{k}}\det(L_A)$.
Let $\mathcal{L}_{N,k}:=\{L\in\mathrm{Sym}^+(N): Z_k(L)>0\}$.

We first express the $k$-DPP in the parameterization using the spectral decomposition, which provides a connection to the exponential family. Recall that an exponential family is said to be non-minimal if its sufficient statistics satisfy a nontrivial affine relation on the sample space, 
while an exponential family is said to be minimal if there is no such relation.

\begin{proposition}[The expression of $k$-DPPs using the spectral decomposition]
\label{prop:1}
Fix $k\in\{1,\ldots,N-1\}$ and $U\in O(N)$.  For
$\theta\in\mathbb{R}^N$, let
$L(\theta)=U\,\mathrm{diag}(e^{\theta_1},\ldots,e^{\theta_N})U^{\top}$.
Then the induced $k$-DPP on $\mathcal{X}_k$ is
\begin{equation}
P^{k}_{L(\theta)}(A)
=\displaystyle\sum_{B\in\mathcal{X}_{k}}\det(U_{A,B})^2\exp\!\Bigg{(}\sum_{i\in B}\theta_i\Bigg{)}
      \Big{/}e_k(e^{\theta_1},\ldots,e^{\theta_N}),
\quad A\in \mathcal{X}_k,
\label{eq:kDPP-explicit}
\end{equation}
where $U_{A,B}$ denotes the $k\times k$ submatrix of $U$ with rows in
$A$ and columns in $B$.  When $U=I_{N}$, it forms a non-minimal exponential family
\[
P^{k}_{L(\theta)}(A)
=\exp\!\Bigg{(}\sum_{i=1}^{N}T_{i}(A)\theta_i-\psi_k(\theta)\Bigg{)},
\;
\psi_k(\theta):=\log e_k(e^{\theta_1},\ldots,e^{\theta_N}), \;
\theta\in\mathbb{R}^{N},
\]
where
\[T(A)=(T_{1}(A),\ldots,T_{N}(A))=(\mathbf{1}_{1\in A},\ldots,\mathbf{1}_{N\in A})\] is a non-minimal sufficient statistic. 
Further, using the identifiable parameterization $\tilde{\theta}=(\theta_{1}-\theta_{N},\ldots,\theta_{N-1}-\theta_{N})^{\top}\in\mathbb{R}^{N-1}$ yields a minimal exponential family with natural parameter $\tilde{\theta}$.
\end{proposition}

\begin{proof}
The formula \eqref{eq:kDPP-explicit} follows by applying the
Cauchy--Binet formula to $\det(L(\theta)_A)$.  The exponential-family form for $U=I_{N}$
follows by setting $\det(U_{A,B})^2=\mathbf{1}_{A=B}$.
Considering $\sum_{i=1}^{N}T_{i}(A)=k$ shows the non-minimality under $\theta$. 
Minimality under $\tilde{\theta}$ follows because any affine relation
$\sum_{i=1}^{N-1}a_iT_i(A)=c\; (A\in\mathcal X_k)$
must be trivial: for each \(j\le N-1\), choose
$R\subseteq\mathcal Y\setminus\{j,N\}$ with $|R|=k-1$, and compare
$R\cup\{j\}$ with $R\cup\{N\}$.  This gives 
$a_j=0$, hence all $a_{j}$ should be zero, which completes the proof.
\end{proof}

\subsection{The Fisher information matrix}
\label{subsec:fisher-N3k2}
% =====================================================================

Proposition~\ref{prop:1} shows that in the diagonal eigenspace rotation case $U=I_{N}$ the
$k$-DPP forms an exponential family,
where $\theta$ is non-identifiable parameterization.
We now make this
picture quantitative 
using the Fisher information matrix
in the smallest nontrivial case $N=3$, $k=2$.
Write $s_i:=e^{\theta_i}$ and
$e_2:=e_2(s_1,s_2,s_3)=s_1s_2+s_1s_3+s_2s_3$.
Proposition~\ref{prop:1} gives
$P^{(2)}_\theta(\{i,j\})=s_is_j/e_2$,
and the Fisher information matrix with respect to $\theta$ is
$G(\theta)=\nabla^2_{\theta}\psi_2(\theta)$ with
$\psi_2(\theta)=\log e_2(e^{\theta_1},e^{\theta_2},e^{\theta_3})$.

Consider the mean parameter.
Differentiating $\psi_2$ once gives the inclusion probabilities
\begin{equation}
\eta_i(\theta):=\partial_i\psi_2(\theta)
=\mathbb{E}_\theta[\mathbf{1}_{i\in A}]
=\frac{s_i(s_j+s_k)}{e_2},
\qquad\{i,j,k\}=\{1,2,3\}.
\label{eq:eta}
\end{equation}
From the cardinality constraint $|A|=2$, we have $\eta_1+\eta_2+\eta_3=2$, which indicates the non-identifiability in $\eta$.

We next consider Fisher information matrix.
Writing $\pi_{ij}:=P^{(2)}_\theta(\{i,j\})=s_is_j/e_2$, a direct
calculation yields
\begin{equation*}
G_{ii}=\eta_i(1-\eta_i),\qquad G_{ij}=\pi_{ij}-\eta_i\eta_j\;(i\ne j),
\end{equation*}
or, equivalently,
\begin{equation}
G(\theta)=\mathrm{diag}(\eta)-\eta\eta^{\top}+\Pi,
\qquad
\Pi=\begin{pmatrix}0&\pi_{12}&\pi_{13}\\\pi_{12}&0&\pi_{23}\\\pi_{13}&\pi_{23}&0\end{pmatrix}.
\label{eq:fisher-matrix}
\end{equation}
This Fisher information matrix with respect to $\theta$ is degenerate because $\theta$ is non-identifiable.
In fact, differentiating $\eta_1+\eta_2+\eta_3=2$ in $\theta_i$ gives
$G(\theta)\,\mathbf{1}_{N=3}=\mathbf{0}_{N=3}$ with $\mathbf{0}_{N=3}$ all-zero vector of size $N$.
Moreover,
$G(\theta)$ has rank exactly $2$: the restriction to
$\mathbf{1}^{\perp}_{N=3}:=\{u:u_1+u_2+u_3=0\}$ is positive definite,
since
$u^{\top}G(\theta)u=\mathrm{Var}_\theta(\sum_i u_i\mathbf{1}_{i\in A})>0$
for $u\ne0$ in $\mathbf{1}^{\perp}_{N=3}$.

Consider the Fisher information matrix at a symmetric point: $s_{1}=s_{2}=s_{3}$.
As an example, consider $\theta=0$ (i.e.\ $s_i=1$). 
Then, we have $\eta_i=\tfrac{2}{3}$ and $\pi_{ij}=\tfrac{1}{3}$, and \eqref{eq:fisher-matrix} reduces to
\begin{equation}
G(0)=\frac{1}{3}\begin{pmatrix}2&-1&-1\\-1&2&-1\\-1&-1&2\end{pmatrix}
=I_{N=3}-\tfrac{1}{3}\mathbf{1}_{N=3}\mathbf{1}^{\top}_{N=3}.
\label{eq:G0}
\end{equation}
Thus $G(0)$ is the orthogonal projector onto $\mathbf{1}^{\perp}_{N=3}$,
with eigenvalues $\{0,1,1\}$.
So, the Fisher metric at the symmetric
point is the canonical metric on the additive quotient vector space $\mathbb{R}^3/\mathbb{R}\mathbf{1}_{3}$.

% =====================================================================
\section{Identifiability issue}
\label{sec:identifiability}
% =====================================================================

We now move to the identifiability issue. 
Proposition \ref{prop:1} exhibits the non-identifiability of $k$-DPPs even when the eigenspace orientation is identity.
The full DPP
is determined by its kernel $L$ up to sign similarity $L\mapsto DLD$
with $D\in T_\pm$~\cite[Theorem~4.1]{Kulesza2012-rt},
where $T_\pm$ is the sign-flip
group $T_\pm:=\{\mathrm{diag}(\varepsilon_1,\ldots,\varepsilon_N):\varepsilon_i\in\{\pm1\}\}$.
Conditioning on the cardinality collapses this picture significantly.
The following theorem shows there are several non-identifibility in the $k$-DPP.

\begin{theorem}[Explicit invariances via spectral transformation]
\label{thm:invariances}
Fix $1 \le k \le N-1$ and let $L \in \mathcal{L}_{N,k}$ with its spectral decomposition $L = U \Lambda U^\top$. For a scale factor $c > 0$, a sign-flip matrix $D \in T_{\pm}$, and an orthogonal matrix $Q \in O(N)$ satisfying $Q \Lambda Q^\top = \Lambda$, let $M$ be the kernel matrix reconstructed via the transformed spectral components:
\begin{equation}
M = (D U Q) (c \Lambda) (D U Q)^\top.
\label{eq:M_transform}
\end{equation}
Then, it holds that $P_{M}^k = P_{L}^k$. This transformation simultaneously characterizes the following three invariances:
\begin{enumerate}
\item \emph{(Scale.)} $\Lambda \mapsto c \Lambda$ for $c > 0$.
\item \emph{(Sign similarity.)} $U \mapsto D U$ for $D \in T_{\pm}$.
\item \emph{(Eigenspace rotation.)} $U \mapsto U Q$ for $Q \in O(N)$ such that $Q \Lambda = \Lambda Q$.
\end{enumerate}
\end{theorem}

\begin{proof}
Expanding~\eqref{eq:M_transform} and using $Q \Lambda Q^\top = \Lambda$, the reconstructed kernel simplifies directly to $M = c D L D$. For any $A \in \mathcal{X}_k$, the principal minor satisfies $\det(M_A) = \det((c D L D)_A) = c^k \det(D_A)^2 \det(L_A) = c^k \det(L_A)$, since the submatrix $D_A$ is diagonal with entries $\pm 1$. Thus, the common factor $c^k$ cancels out in the numerator and denominator of the $k$-DPP definition, yielding $P_M^k(A) = P_L^k(A)$ for all $A \in \mathcal{X}_k$.
\end{proof}

Theorem~\ref{thm:invariances} does not necessarily exhaust the identifiability gap between the full DPP and the $k$-DPP.  The following theorem shows that the $k$-DPP exhibits (nontrivial) additional continuous unidentifiability whenever $\binom{N}{k}$ is strictly smaller than $N(N+1)/2$. For $L\in\mathcal L_{N,k}$ and $H\in\mathrm{Sym}(N)$, define the log-likelihood score in direction $H$ by
\begin{equation}
\mathcal{S}_H(A;L)
\;:=\;
\frac{d}{dt}\log P^k_{L+tH}(A)\Big|_{t=0},
\quad A\in \mathcal{X}_k.
\label{eq:score}
\end{equation}

\begin{theorem}[Dimension of non-identifiable directions]
\label{thm:dim}
Assume $L\in \mathrm{Sym}^{++}(N,\mathbb{R})\cap \mathcal{L}_{N,k}$, where $\mathrm{Sym}^{++}(N,\mathbb{R})$ is the set of real-valued, symmetric, positive definite $N\times N$ matrices.
The score vanishes in direction $H$ for every $A\in \mathcal{X}_{k}$ if and only if
\begin{equation}
\mathrm{tr}(L_A^{-1}H_A)\;=\;\text{const.}\quad\text{for all }A\in \mathcal{X}_k.
\label{eq:invisible}
\end{equation}
Further, the space of the non-identifiable directions has a dimension bound
\begin{equation}
\dim\bigl\{H\in\mathrm{Sym}(N):\mathcal{S}_H(\,\cdot\,;L)=0\bigr\}
\;\ge \;
\frac{N(N+1)}{2}-\binom{N}{k}+1.
\label{eq:dim-formula}
\end{equation}
In particular, if
\[\binom{N}{k}<\frac{N(N+1)}{2},\] 
the space of the non-identifiable directions has dimension 
strictly larger than the one-dimensional scale cone $\{cL\;:\;c>0\}$.
\end{theorem}
\begin{proof}

We begin with the score formula
\begin{align}
  \mathcal{S}_H(A;L)  \;=\;
\mathrm{tr}(L_A^{-1}H_A)
\;-\;
\mathbb{E}_{S\sim P^k_L}\!\bigl[\mathrm{tr}(L_S^{-1}H_S)\bigr]
    \label{eq: score formula}
\end{align}
This follows by differentiating
$\log P^k_{L+tH}(A)=\log\det((L+tH)_A)-\log Z_k(L+tH)$:
\[
\frac{d}{dt}\log\det((L+tH)_A)\Big|_{t=0}
=\mathrm{tr}(L_A^{-1}H_A),
\qquad
\frac{d}{dt}\log Z_k(L+tH)\Big|_{t=0}
=\mathbb{E}_{S\sim P^k_L}\!\bigl[\mathrm{tr}(L_S^{-1}H_S)\bigr].
\]
This score formula \eqref{eq: score formula} proves the first part.

Next, we will show \eqref{eq:dim-formula}.
Let $m:=N(N+1)/2$ and
$\Phi(H):=( \mathrm{tr}(L^{-1}_{A}H_{A}) )_{A\in \mathcal{X}_{k}}$, where $\Phi$ forms a linear map from $\mathrm{Sym}(N)$ to $\mathbb{R}^{\mathcal{X}_{k}}$.
From  \eqref{eq: score formula}, we have
\[V:=\{H\in\mathrm{Sym}(N):\mathcal{S}_H(\,\cdot\,;L)=0\bigr\}=\Phi^{-1}(\mathbb{R}\mathbf{1}_{\mathcal{X}_{k}})\]
with $|\mathcal{X}_{k}|$-dimensional all-one vector $\mathbf{1}_{\mathcal{X}_{k}}$.
The rank-nullity formula gives
\[
\mathrm{dim}(\mathrm{Ker}(\Phi))=
\underbrace{\mathrm{dim}(\Phi^{-1}(\mathbb{R}^{\mathcal{X}_{k}}))}_{=m}-\mathrm{rank}(\Phi)
\]
For $\tilde{\Phi}:=\Phi |_{\Phi^{-1}(\mathbb{R}\mathbf{1}_{\mathcal{X}_{k}})}$,
the rank-nullity formula gives
\[
\mathrm{dim}(V)=\mathrm{dim}(\Phi^{-1}(\mathbb{R}\mathbf{1}_{\mathcal{X}_{k}}))
=\mathrm{dim}(\mathrm{Ker}(\tilde{\Phi}))+\underbrace{\mathrm{rank}(\tilde{\Phi})}_{=1}.
\]
Since $\mathrm{Ker}(\Phi)=\mathrm{Ker}(\tilde{\Phi})$, it follows that 
\[
\mathrm{dim}(V)=m-\mathrm{rank}(\Phi)+1.
\]
Together with the upper bound $\mathrm{rank}(\Phi)\le \mathrm{dim}(\mathbb{R}^{\mathcal{X}_{k}})=\binom{N}{k}$, this gives
\[
\mathrm{dim}(V)\ge m-\binom{N}{k}+1,
\]
which completes the proof.
\end{proof}

The strict inequality
\[
\binom{N}{k}>\frac{N(N+1)}{2}
\]
already occurs for \(N=7\) and \(k\in\{3,4\}\), since
\[
\binom{7}{3}=\binom{7}{4}=35>28=\frac{7\cdot 8}{2}.
\]
In contrast, the nearest threshold example is \(N=6\), \(k=3\), where
\[
\binom{6}{3}=20<21=\frac{6\cdot 7}{2}.
\]
In this case, the lower bound gives at least a two-dimensional
space of non-identifiable directions.

\begin{example}[Nontrivial continuous non-identifiable directions]
Fix $L=I_N$. For $\rho \in \left(-\frac{1}{N-1}, 1\right)$, let $H_\rho = -\rho I_N + \rho 1_N 1_N^\top$.
Since $L_A = I_k$, the local score strictly vanishes: $\mathrm{tr}(L_A^{-1}H_{\rho, A}) = \mathrm{tr}(-\rho I_k + \rho 1_k 1_k^\top) = 0$ for all $A \in \mathcal{X}_k$, satisfying the condition in Theorem 2.
Globally, the kernel $L_\rho = L + H_\rho$ yields principal minors $\det((L_\rho)_A) = (1-\rho)^{k-1}(1+(k-1)\rho)$.
This is independent of $A$, meaning every admissible $L_\rho$ induces the identical uniform $k$-DPP.
Because distinct values of $\rho$ yield distinct eigenvalue ratios, no two $L_\rho$ can be related by the trivial transformations of Theorem 1, explicitly confirming a continuous non-identifiable fiber.
\end{example}

% =====================================================================
\section{Conclusion}
% =====================================================================

We have proposed the geometric analysis using the spectral decomposition of
DPPs, centered on the contrast between the full DPP and the
$k$-DPP.  The spectral decomposition $L=U\Lambda U^{\top}$ separates
two statistically distinct components: the eigenvalues $\Lambda$
govern the cardinality profile, while the eigenspace orientation $U$
governs the conditional distribution within each stratum.

After conditioning on $|A|=k$, the induced $k$-DPP retains only relative spectral information, i.e., the common scale becomes unidentifiable and the eigenspace rotation is identifiable only through squared minors. Three explicit non-identifiability (scale, sign similarity, eigenspace rotation) extend beyond the single sign-similarity non-identifiability that characterizes the full DPP. A dimension-counting theorem shows that an additional non-trivial continuous non-identifiability exists in sharp contrast to the discrete finite non-identifiability of the full DPP. 

At the Fisher metric level, the common-shift direction
$\mathbb{R}\mathbf{1}_{N}$ is exactly the null direction, and the metric at the symmetric point equals the canonical projector onto the additive quotient space. This provides a mathematically transparent foundation for further analyses of constrained parametrization, curvature, and statistical inference in DPPs.

The problem of learning the DPP kernel matrix from data is of central importance in the machine learning community~\cite{gillenwater2014,Kawashima2023-kv}. As a direct extension of this work, developing kernel learning algorithms that explicitly leverage the spectral decomposition constitutes a promising direction for future research

\section*{Acknowledgment}
Part of this work is supported by JSPS KAKENHI No.~JP23K24909, JP24H00247,
JP25H01494, JP26K02871 and by MEXT STAR-E NEXT Project (No.~JPJ013735).

%\bibliography{spec_DPP}

\newpage 

% =====================================================================
\appendix
% =====================================================================

\section{Fisher information expression}
\label{app:fisher-k1-kN1}

For completeness, we record the Fisher information of the diagonal
$k$-DPP for the two extremal cases $k=1$ and $k=N-1$ at $N=3$.

\textbf{Case $k=1$, $N=3$.}
The sample space is $X_1=\{\{1\},\{2\},\{3\}\}$ and
$P^{(1)}_\theta(\{i\})=e^{\theta_i}/(e^{\theta_1}+e^{\theta_2}+e^{\theta_3})$,
which is the standard categorical (softmax) model.  Writing
$p_i:=P^{(1)}_\theta(\{i\})$, the Fisher information matrix is the
well-known expression
\[
G^{(1)}(\theta)=\mathrm{diag}(p)-pp^{\top},
\]
with $G^{(1)}(\theta)\mathbf{1}_{N=3}=0$ and $\mathrm{rank}\,G^{(1)}=2$.
At the symmetric point $\theta=0$, $p_i=1/3$ and
\[
G^{(1)}(0)=\frac{1}{3}\begin{pmatrix}2&-1&-1\\-1&2&-1\\-1&-1&2\end{pmatrix}
=I_3-\tfrac{1}{3}\mathbf{1}_{N=3}\mathbf{1}_{N=3}^{\top},
\]
identical in form to \eqref{eq:G0}, confirming the symmetry $k\leftrightarrow N-k$.

\textbf{Case $k=N-1=2$, $N=3$.}
This is the case treated in Section~\ref{subsec:fisher-N3k2}, and the
Fisher matrix is given by \eqref{eq:fisher-matrix}--\eqref{eq:G0}.

\textbf{Symmetry $k\leftrightarrow N-k$.}
For the diagonal $k$-DPP with $N=3$, the maps $k=1$ and $k=2$ are
related by the complementation $A\mapsto Y\setminus A$, and indeed
the two Fisher matrices coincide at the symmetric point $\theta=0$.
This reflects the general symmetry
$P^{k}_{L(\theta)}(A)=P^{N-k}_{L(-\theta)}(Y\setminus A)$,
which implies $G^{(k)}(\theta)=G^{(N-k)}(-\theta)$ for the diagonal
$k$-DPP.

\section{Inclusion probabilities via exterior power}
\label{app:kthInclusion}

The following is a complementary observation that relates the
spectral decomposition to the structure of inclusion
probabilities of the \emph{unconditional} DPP.

To connect with the marginal kernel, recall $K=L(L+I)^{-1}$.  The
matrices $L$ and $K$ share eigenvectors with eigenvalues
$\mu_i=\lambda_i/(1+\lambda_i)$, so $K=UMU^{\top}$,
$M=\mathrm{diag}(\mu_1,\ldots,\mu_N)$.  The passage from $L$ to $K$
preserves the orbital component and acts only on the spectral
component.

Distinct from the $k$-DPP is the family of $k$-th order
\emph{inclusion probabilities} of the unconditional DPP
$A\sim\mathrm{DPP}(L)$:
\begin{equation}
\rho_k(S):=P(S\subseteq A)=\det(K_S),
\qquad S\subseteq Y,\ |S|=k.
\label{eq:inclusion}
\end{equation}
Here $k=|S|$ is simply the order of the inclusion event, not a
conditioning value; $\rho_k(S)$ depends on the full spectrum of $L$.

\begin{proposition}[$k$-th inclusion probabilities via exterior power]
\label{prop:exterior}
Fix $k\in\{1,\ldots,N\}$ and let $K=UMU^{\top}$.  Then
$\bigwedge^k K=(\bigwedge^k U)(\bigwedge^k M)(\bigwedge^k U)^{\top}$,
and for each $S=\{s_1,\ldots,s_k\}$ with
$v_S=e_{s_1}\wedge\cdots\wedge e_{s_k}$,
\begin{equation}
\rho_k(S)=\langle v_S,\,(\textstyle\bigwedge^k K)\,v_S\rangle.
\label{eq:lifted}
\end{equation}
The spectral--orbital decomposition persists: spectral weights
$\mu_{i_1}\cdots\mu_{i_k}$ appear as diagonal entries of
$\bigwedge^k M$, while the orbital component acts by the orthogonal
rotation $\bigwedge^k U$ on $\bigwedge^k\mathbb{R}^N$.
\end{proposition}

\begin{proof}
Functoriality of the exterior power gives the factorization.  Since $M$ is diagonal, $\bigwedge^k M$ is diagonal with entries $\mu_{i_1}\cdots\mu_{i_k}$. The principal minor $\det(K_S)$ equals the diagonal entry of $\bigwedge^k K$ at $v_S$, which gives
\eqref{eq:lifted}.
\end{proof}

\begin{remark}
Proposition~\ref{prop:exterior} concerns the unconditional DPP, not the conditional $k$-DPP. The quantities $\rho_k(S)=P(S\subseteq A)$ are inclusion probabilities summed over all cardinalities compatible with $S\subseteq A$; they are not probabilities on the stratum $X_k$. 
While $\rho_k$ depends on the full spectrum of $L$, the $k$-DPP probability $P_L^k$ of Proposition~1 is invariant under a common scaling of the eigenvalues. Equivalently, in the log-eigenvalue coordinates $\theta_i=\log\lambda_i$, it depends only on the common-shift equivalence class $[\theta]\in \mathbb R^N/\mathbb R\mathbf 1_N.$
Thus Proposition~2 and Proposition~1 describe different objects: the former concerns fixed-order inclusion probabilities of the full DPP, whereas the latter concerns the conditional law on a fixed cardinality stratum.
\end{remark}

For a projection $2$-DPP on $N=4$, the marginal kernel $K=VV^{\top}$ is associated with the $2$-vector $\omega=v_1\wedge v_2\in\bigwedge^2\mathbb{R}^4$ with Pl\"ucker coordinates $p_{ij}=\det(V_{ij})$.  Decomposability ($\omega\wedge\omega=0$) yields the classical Pl\"ucker relation $p_{12}p_{34}-p_{13}p_{24}+p_{14}p_{23}=0$, which translates back to
\[
\sqrt{\rho_2(\{1,2\})\rho_2(\{3,4\})}
\pm\sqrt{\rho_2(\{1,3\})\rho_2(\{2,4\})}
\pm\sqrt{\rho_2(\{1,4\})\rho_2(\{2,3\})}=0.
\]
This exposes the nonlinear algebraic constraint on inclusion probabilities that the ambient linearization of Proposition~\ref{prop:exterior} makes visible. 
\end{document}